\documentclass[runningheads, 11pt, a4paper, citeauthoryear, envcountsame]{llncs}

\usepackage{appendix}
\usepackage{amsmath, amssymb, amsfonts,latexsym, amscd,mathrsfs}
\usepackage{graphicx}
\usepackage{multirow}
\usepackage[usenames]{color}
\usepackage[hypertexnames=false, colorlinks=true,citecolor=magenta,  linkcolor = blue, urlcolor = blue, filecolor = blue]{hyperref}
\usepackage[all]{hypcap}
\usepackage{caption}
\usepackage{subcaption}

\usepackage{algorithm} % has to be loaded AFTER hypperrf
\usepackage{algorithmic}

\usepackage[backend=bibtex, natbib=true, doi=false, isbn=false,url=false, arxiv=false, eprint=false, maxcitenames=2,  sorting=none]{biblatex}
\renewcommand{\phi}{\varphi}

\newcommand{\norm}[1]{\left\lVert#1\right\rVert}

\spnewtheorem*{remark}{Remark}{\itshape}{\rmfamily}

\begin{document}
%% squeez some space
\setlength{\abovedisplayskip}{0.5em}
\setlength{\belowdisplayskip}{0.5em}
\setlength{\belowcaptionskip}{0em}
\setlength{\abovecaptionskip}{0.5em}

\mainmatter  % start of an individual contribution

\title{Differentially-Private Logistic Regression for Detecting Multiple-SNP Association in GWAS Databases}

\titlerunning{Differentially-Private GWAS Analysis}

\author{Fei Yu\inst{1}%
\thanks{This research was partially supported by NSF Awards EMSW21-RTG and BCS-0941518 to the Department of Statistics at Carnegie Mellon University.  %This study makes use of data generated by the Wellcome Trust Case-Control Consortium. A full list of the investigators who contributed to the generation of the data is available from http://www.wtccc.org.uk.}%
}
 \and Michal Rybar\inst{2} \and Caroline Uhler\inst{2} \thanks{This research was partially performed while visiting the Simons Institute for the Theory of Computing. }
\and Stephen E. Fienberg\inst{1}
}
%


\institute{
Carnegie Mellon University, Pittsburgh, PA 15213, USA \\
\email{feiy@stat.cmu.edu, fienberg@stat.cmu.edu}
\and 
Institute of Science and Technology Austria, Am Campus 1, 3400 Klosterneuburg, Austria \\
\email{michal.rybar@ist.ac.at, caroline.uhler@ist.ac.at} 
}

\maketitle

%%%%%%%%%%%%%%%%%%%%%%%%%%%%%%%%%%%%%%%%%%%%%%%%%%%%%%%%%%%%%%%%%%%
%%%%%%%%%%%%%%%%%%%%%%%%%%%%%%%%%%%%%%%%%%%%%%%%%%%%%%%%%%%%%%%%%%%
%%%%%%%%%%%%%%%%%%%%%%%%%%%%%%%%%%%%%%%%%%%%%%%%%%%%%%%%%%%%%%%%%%%
\begin{abstract}
Following the publication of an attack on genome-wide association studies (GWAS)  data proposed by Homer et al., considerable attention has been given to developing methods for releasing GWAS data in a privacy-preserving way. Here, we develop an end-to-end differentially private method for solving regression problems with convex penalty functions and selecting the penalty parameters by cross-validation. In particular, we focus on penalized logistic regression with elastic-net regularization, a method  widely used to in GWAS analyses to identify disease-causing genes.  We show how a differentially private procedure for penalized logistic regression with elastic-net regularization can be applied to the analysis of GWAS data and evaluate our method's performance.

\end{abstract}

%%%%%%%%%%%%%%%%%%%%%%%%%%%%%%%%%%%%%%%%%%%%%%%%%%%%%%%%%%%%%%%%%%%
%%%%%%%%%%%%%%%%%%%%%%%%%%%%%%%%%%%%%%%%%%%%%%%%%%%%%%%%%%%%%%%%%%%
%%%%%%%%%%%%%%%%%%%%%%%%%%%%%%%%%%%%%%%%%%%%%%%%%%%%%%%%%%%%%%%%%%%
\section{Introduction}

\subsection{Genetic data privacy background}
The goal of a genome-wide association study (GWAS) is to identify genetic variations  associated with a disease. Typical GWAS databases contain information on hundreds of thousands of single nucleotide polymorphisms (SNPs) from thousands of individuals. The aim of GWAS is to find associations between SNPs and a certain phenotype, such as a disease. A particular phenotype is usually the result of complex relationships between multiple SNPs, making GWAS a very high-dimensional problem.

Recently, penalized regression approaches have been applied to GWAS to overcome the challenges caused by the high-dimensional nature of the data. A popular approach consists of a two-step procedure. In the first step, all SNPs are screened and a subset is selected based on a simple $\chi^2$-test for association between each single SNP and the phenotype. In the second step, the selected subset of SNPs is tested for multiple-SNP association using penalized logistic regression. Elastic-net regularization, which imposes a combination of $\ell_1$ and ridge penalties, has been shown to be a competitive method for GWAS (e.g. \cite{Austin2013, Cho2009}).

For many years, researchers believed that  releasing  statistics of SNPs aggregated from thousands of individuals would not compromise the participants' privacy. Such a belief came under challenge with the publication of an attack proposed by \citet{Homer2008}. This publication drew widespread attention. As a consequence,  NIH removed all aggregate SNP data from open-access databases \citep{Couzin2008, Zerhouni2008} and instituted an elaborate approval process for gaining access to aggregate genetic data. This NIH action in turn spurred interest in the development of methods for confidentiality protection of GWAS databases.

\subsection{Differentially private methods for solving regression problems}
The  approach of differential privacy, introduced by the cryptographic community (e.g. \citet{Dwork2006}), provides privacy guarantees that protect GWAS databases against arbitrary external information. Building on such notion, \citet{Uhler2013, Johnson2013, Yu2014}  proposed new methods for selecting a subset of SNPs in a differentially-private manner. These approaches enable us to perform the first step in the two-step procedure for identifying the relevant SNPs in a GWAS  without compromising the study participants' privacy. 
The second step of the two-step procedure would involve performing penalized logistic regression with elastic-net regularization ($l_1$ and $l_2$ penalties) on the selected subset of SNPs in a differentially private manner.  \citet{Kifer2012} proposed an objective function perturbation mechanism that releases the coefficients of a convex risk minimization problem with convex penalties and satisfies differential privacy. We can use this method to perform logistic regression with elastic-net  regularization in a differentially private way. 
 
The performance of penalized logistic regression approaches depends heavily on the choice of regularization parameters. Selection of these regularization parameters is usually done via cross-validation. \citet{Chaudhuri2013} proposed a differentially-private procedure for choosing the regularization parameters  based on a stability argument. However, the method proposed by \citet{Chaudhuri2013} only works on differentiable penalty functions, such as the $\ell_2$ penalty, and it cannot be applied to elastic-net regularization or lasso.

In Section~\ref{sec_2}, we extend the stability-based method for selecting the regularization parameters developed by \citet{Chaudhuri2013} so that it is applicable to any convex penalty function, including the elastic-net penalty. By combining this new result and the objective function perturbation mechanism proposed by \citet{Kifer2012},  we are able to carry out a privacy-preserving penalized logistic regression analysis. In Section~\ref{sec_3}, we demonstrate how to implement the full objective function perturbation mechanism with cross-validation based on the results by \citet{Chaudhuri2011} and \citet{Kifer2012}. In particular, we provide the exact form of the random noise used in the objective function perturbation mechanism. Furthermore, we show that, under a slightly stronger condition, we can perturb the objective function by an alternative form of noise---the multivariate Laplace noise---and thereby obtain more accurate results. In Section~\ref{sec_4}, we show how to apply our results to develop an end-to-end differentially private penalized logistic regression method with elastic-net penalty and cross-validation for the selection of the penalty parameters. Finally, in Section~\ref{sec_5}, we demonstrate how well this end-to-end differentially private method performs on a GWAS data set.

%%%%%%%%%%%%%%%%%%%%%%%%%%%%%%%%%%%%%%%%%%%%%%%%%%%%%%%%%%%%%%%%%%%
%%%%%%%%%%%%%%%%%%%%%%%%%%%%%%%%%%%%%%%%%%%%%%%%%%%%%%%%%%%%%%%%%%%
%%%%%%%%%%%%%%%%%%%%%%%%%%%%%%%%%%%%%%%%%%%%%%%%%%%%%%%%%%%%%%%%%%%

\section{Differentially-private penalized regression}
\label{sec_2}

We start by reviewing the concept of differential privacy. Let $\mathcal{D}$ denote the set of all data sets.  Let $D, D'\in\mathcal{D}$ denote two data sets that differ in one individual only. We denote this by $D \sim D'$.
 
\begin{definition}[differential privacy]
A randomized mechanism $\mathcal{K}$ is $\epsilon$-differentially private
if, for all  $D \sim D'$  and 
for any measurable set $S \subset \mathbb{R}$, 

\[
\dfrac{\mathbb{P}(\mathcal{K}(D) \in S )  }{ \mathbb{P}(\mathcal{K}(D') \in S )  } \le e^\epsilon
.\]
$\mathcal{K}$ is   $(\epsilon, \delta)$-differentially private if, for all  $D \sim D'$  and 
for any measurable set $S \subset \mathbb{R}$, 
\[
\mathbb{P}(\mathcal{K}(D) \in S ) \le e^\epsilon \mathbb{P}(\mathcal{K}(D') \in S ) + \delta
.\]
\end{definition}

Let $l: \mathbb{R}^s \times \mathcal{D} \to \mathbb{R}$ denote the loss function, $r:  \mathbb{R}^s \to \mathbb{R}$ a regularization function, and $h: \mathbb{R}^s \times \mathcal{D} \to \mathbb{R}$ the validation function. Let $T\in\mathcal{D}^n$ be a training data set of size $n$ drawn from $\mathcal{D}$ and $V\in\mathcal{D}^m$ a validation data set of size $m$ also drawn from $\mathcal{D}$. Let $b \in  \mathbb{R}^s$ denote the noise used to perturb the regularized loss function. Then we denote by $\mathcal{T} (\lambda, \epsilon; T, l, r, b)$ the differentially private procedure to produce parameter estimates from the training data $T$ given the regularization parameter $\lambda$, the privacy budget $\epsilon>0$, the loss function $l$,  the regularization function $r$, and the random noise $b$. We score a vector of regression coefficients resulting from the random procedure $\mathcal{T} (\lambda, \epsilon; T, l, r, b)$ using the validation data $V$ and the validation score function $q(\theta, V) = -\frac{1}{m}\sum_{d \in V}^{} h(\theta; d).$

\begin{definition}[$(\beta_1, \beta_2, \delta)$-stability. \citet{Chaudhuri2013}]
A validation score function $q$ is said to be $(\beta_1, \beta_2, \delta)$-stable with respect to a training  procedure  $\mathcal{T}$, the candidate regularization parameters $\Lambda$, and the privacy budget $\epsilon$, if there exists  $E\subset\mathbb{R}^s$ such that $\mathbb{P}(b \in E) \ge 1 - \delta$, and when $b \in E$, the following conditions hold:

\begin{enumerate}
\item {\bf Training stability:} for all $\lambda \in \Lambda$, for all validation data sets $V\in\mathcal{D}^m$, and all training dataset $T, T'\in\mathcal{D}^n$ with $T\sim T'$,  
 $$
|~ q(\mathcal{T} (\lambda, \epsilon; T, l, r, b), V) - q(\mathcal{T} (\lambda, \epsilon; T', l, r, b), V)  ~| \le \frac{\beta_1}{n}. $$

\item {\bf Validation stability:} for all $\lambda \in \Lambda$, for all training data sets $T\in\mathcal{D}^n$, and all validation data sets $V, V'\in\mathcal{D}^m$ with $V \sim V'$, 
 $$ |~ q(\mathcal{T} (\lambda, \epsilon; T, l, r, b), V) - q(\mathcal{T} (\lambda, \epsilon; T, l, r, b), V')  ~| \le \frac{\beta_2}{m}.$$

\end{enumerate}
\end{definition}

\citet{Chaudhuri2013} gave conditions under which a validation score function is $(\beta_1, \beta_2, \delta)$-stable when the regularization function is differentiable and showed that as long as the validation score function $q$ is $(\beta_1, \beta_2, \delta)$-stable for some $\beta_1, \beta_2, \delta > 0$ with respect to the procedure  $\mathcal{T}$, candidate regularization parameters $\Lambda$, and privacy budget $\epsilon$, we can choose the best regularization parameter in a differentially private manner using Algorithm~1 and Algorithm~2 in \citet{Chaudhuri2013}. In Theorem~\ref{thm:stability_dp}, we specify the conditions under which a validation score function is $(\beta_1, \beta_2, \delta)$-stable for a general convex regularization function.

In the following, we combine the regularization function and the regularization parameters to form a vector of candidate regularization functions $r =  (r_1, \dots, r_t)$. Then, selecting the regularization parameters is equivalent to selecting a linear combination of $r_i$'s in $r$.

\begin{theorem}
\label{thm:stability_dp}
Let $r = (r_1, \dots, r_t)$ be a vector of convex regularization functions with $r_i: \mathbb{R}^s \to \mathbb{R}$ that are minimized at 0. 
Let $\Lambda = \{\lambda_1, \dots, \lambda_k \}$ be a collection of regularization vectors, where $\lambda_i$ is a $t$-dimensional vector of 0's and 1's. 
We denote by
$c_{min} := \sup_c\{ \forall \lambda \in \Lambda, \lambda^T r \text{ is $c$-strong convex}\}.$
Let $h(\theta; d)$ be a validation score that is non-negative and $\kappa$-Lipschitz in $\theta$. We denote  $\max_{d \in \mathcal{D}, \theta \in \mathbb{R}^s} h(\theta; d)$ by $h^{*}$.  In addition, let
 $l(\theta; d)$ be a convex loss function that is $\gamma$-Lipschitz in $\theta$. Finally, let $\xi\in\mathbb{R}$ such that $\mathbb{P}(\norm{b}_2 > \xi)  \le \delta / k$ for some $\delta \in (0, 1)$.
Then the validation score 
$q(\theta, V) = -\frac{1}{m} \sum\limits_{d \in V} h(\theta; d)$ is $(\beta_1, \beta_2, \delta/k)$-stable  with respect to $\mathcal{T}$, $\epsilon$ and $\Lambda$, where 
$$\mathcal{T}(\lambda, \epsilon; T, l, r, b ) := \arg\min_\theta L(\theta; \lambda, \epsilon),$$  with
$$L(\theta; \lambda, \epsilon) = \frac{1}{n} \sum\limits_{d\in T}^{} l (\theta; d) + \lambda^T r(\theta) + \frac{\max\{0,\; c^* - c_{\min}\}}{2}\norm{\theta}_2^2 + \frac{\phi}{\epsilon n}b^T\theta, $$ 
\[
\begin{split}
\beta_1 = \frac{2\gamma \kappa}{\max\{c^*, c_{min}\}}, 
\qquad 
\beta_2 = \min \left\{ h^*, \frac{\kappa}{\max\{c^*, c_{min}\}} \left( \gamma + \frac{\phi \xi}{\epsilon n} \right)  \right\}.\\
\end{split}
\]
\end{theorem}

\begin{proof}
See \ref{pf:stability_dp}.
\qed
\end{proof}

Note that choosing 
$r(\theta) = \left( \frac{\lambda_1}{2} \norm{\theta}_2^2, \dots, \frac{\lambda_k}{2}  \norm{\theta}_2^2 \right),$
with $\Lambda = \{e_1, \dots, e_k\},$ where $e_i$ is a $k$-dimensional vector that is 1 in the $i$th entry and 0 everywhere else, results in Theorem 4 in \citet{Chaudhuri2013}. Thus, Theorem \ref{thm:stability_dp} generalizes Theorem 4 in \citet{Chaudhuri2013}.

The term $\frac{\max\{0,\; c^* - c_{\min}\}}{2}\norm{\theta}_2^2$ in Theorem \ref{thm:stability_dp} ensures that $L(\theta; \lambda, \epsilon) $ is at least $c^*$-strongly convex. This is an essential condition for ensuring  that our objective function perturbation algorithm (Algortihm \ref{algthm:objective_perturbation}) is differentially private. The value of $\xi$ in Theorem \ref{thm:stability_dp} depends on the distribution of the perturbation noise $b$. In Section~\ref{sec_3}, we analyze two different distributions for the perturbation noise.

\section{Distributions for the perturbation noise}
\label{sec_3}

% and the objective function perturbation algorithm
\citet{Chaudhuri2011} and \citet{Kifer2012} showed that using  perturbation noise $B_2$ with density function  
$$f_{B_2}(b) \propto \exp\left(-\frac{\norm{b}_2}{2}\right)$$
in the procedure $\mathcal{T}(\lambda, \epsilon; T, l, r, B_2)$ produces $\epsilon$-differentially private parameter estimates. In this section, we describe an efficient method for generating such perturbation noise.  Furthermore, we show that under slightly stronger conditions the procedure $\mathcal{T}(\lambda, \epsilon; T, l, r, B_1)$ is differentially private  when we use perturbation noise $B_1$ with density function 
$$f_{B_1}(b) \propto \exp\left(-\frac{\norm{b}_1}{2}\right),$$
which is simpler to generate than perturbation noise of the form $B_2$.

%%%%%%%%%%%%%%%%%%%%%%%%%%%%%%%%%%%%%%%%%%%%%%%%%%%%%%%%%%%%%%%%%%%
\begin{proposition}
\label{thm:b_dist_exact}
The random variable $X = \frac{W}{\norm{W}_2} Y$, where $W \sim \mathcal{N}(0, I_s)$ and $Y \sim \chi^2(2s)$, has density function $f_X(x) \propto  \exp\left( - \frac{ \norm{x}_2 }{2 }\right)$.
\end{proposition}

\begin{proof} 
See Appendix \ref{pf:b_dist_exact}.
\qed
\end{proof}

This result shows that $B_2 \sim \frac{W_s}{\norm{W_s}_2} Y_{2s}$, with  $W_s \sim \mathcal{N}(0, I_s)$ and $Y_{2s} \sim \chi^2(2s)$. On the other hand, $B_1$ can be viewed as the joint distribution of $s$ independent Laplace random variables with mean $=0$ and scale $=2$. In order to specify the stability parameter $\beta_2$ in Theorem~\ref{thm:stability_dp}, we need to find $\xi\in\mathbb{R}$ such that $P(\norm{b}_2 \ge \xi) \le \delta / k$. The following propositions enable us to find $\xi$ for the perturbation noise $B_1$ and $B_2$.

\begin{proposition}
\label{thm:b1_upper_bound}
$\mathbb{P}\left(  \norm{B_1}_1 \ge 2s\log(sk/\delta) \right) \le \delta / k$. 
\end{proposition}
\begin{proof}
See Lemma 17 in \citet{Chaudhuri2013}. 
\qed
\end{proof}

\begin{proposition}
\label{thm:b2_upper_bound}
$\mathbb{P}\left( \norm{B_2}_2 \ge \left(\sqrt{s} + \sqrt{\log(k/\delta)}\right)^2 + \log(k/\delta) \right) \le \delta / k$.
\end{proposition}
\begin{proof}
Note that $\norm{B_2}_2 = \norm{\frac{W_{s}}{\norm{W_{s}}_2} Y_{2s}}_2   = Y_{2s}$, where $Y_{2s} \sim \chi^2(2s)$. The proof is completed by invoking Lemma 1 in \citet{Laurent2000}.
\qed
\end{proof}

Because $P(\norm{B_1}_1 \ge \xi) \ge P(\norm{B_1}_2 \ge \xi)$, Proposition \ref{thm:b1_upper_bound} and Proposition \ref{thm:b2_upper_bound} enable us to find $\xi\in\mathbb{R}$ such that $P(\norm{b}_2 \ge \xi) \le \delta / k$. When the density function of $b$ is 
$f(b) \propto \exp\left( \frac{\norm{b}_1}{2} \right),$
 then by Proposition \ref{thm:b1_upper_bound}, 
$\xi = 2s\log(sk/\delta).$ 
When the density function of $b$ is $f(b) \propto \exp\left( \frac{||b||_2}{2}\right),$
 then by Proposition \ref{thm:b2_upper_bound},
$\xi =  \left(\sqrt{s} + \sqrt{\log(k/\delta)}\right)^2 + \log(k/\delta).$

%%%%%%%%%%%%%%%%%%%%%%%%%%%%%%%%%%%%%%%%%%%%%%%%%%%%%%%%%%%%%%%%%%%

Algorithm \ref{algthm:objective_perturbation} below is a reformulation of Algorithm 1 in \citet{Kifer2012},  i.e., the differentially private objective function optimization algorithm,  and it incorporates the alternative perturbation noise. The objective function is formulated in such a way that it is compatible with the regularization parameter selection procedure described in Theorem \ref{thm:stability_dp}.

\begin{algorithm}[]
\caption{Generalized Objective Perturbation Mechanism}
\label{algthm:objective_perturbation}
\begin{algorithmic}[1]
\REQUIRE Dataset $D = \{d_1, \dots, d_n\}$; a convex domain $\Theta \subset \mathbb{R}^s$;  privacy parameter $\epsilon$; $\lambda$-strongly convex regularizer $r$; convex loss function $l( \theta; d )$ with rank-1  continuous Hessian $\nabla^2 l(\theta; d)$, an upper bound $c$ on the maximal singular value of $\nabla^2 l(\theta; d)$ and upper bounds $\kappa_j$ on $\norm{\nabla l(\theta; d)}_j$ for $j \in \{1, 2\}$ that hold for all $d\in D$ and all $\theta \in \Theta$. It is also required that  $\phi \ge 2 \kappa_j $ and $\lambda \ge \frac{c }{n \left( e^{\epsilon / 4} - 1\right)}$.

\ENSURE A differentially-private parameter vector $\theta^*$. \\[0.5em]

\STATE Sample $b \in \mathbb{R}^s$ according to noise distribution $B_j$,  $j \in \{1, 2\}$.

\RETURN 
$\theta^* = \arg\min_{\theta} L(\theta; D, \lambda, b)$,
where 
\[
L(\theta; D, \lambda, b)= \frac{1}{n}\sum\limits_{d \in D}l(\theta; d)  + r(\theta) + \frac{\phi}{\epsilon n }b^T\theta.
\]

\end{algorithmic}
\end{algorithm}

%%%%%%%%%%%%%%%%%%%%%%%%%%%%%%%%%%%%%%%%%%%%%%%%%%%%%%%%%%%%%%%%%%%

\begin{theorem}
\label{thm:objective_perturbation}
Algorithm \ref{algthm:objective_perturbation} is $\epsilon$-differentially private.
\end{theorem}
\begin{proof} 
See~\ref{proof_alg}.
\qed
\end{proof}

\subsection{Comparison of the performance of Algorithm \ref{algthm:objective_perturbation} under different noise distributions }

Note that we can always upper bound $\norm{\nabla l(\theta; d)}_2$ by $\norm{\nabla l(\theta; d)}_1$ and hence $\kappa_2\leq\kappa_1$ in Algorithm~\ref{algthm:objective_perturbation}. However, as we show in this section, results from Algorithm \ref{algthm:objective_perturbation} are more accurate when sampling noise from $B_1$ compared to $B_2$. To compare the performance of Algorithm \ref{algthm:objective_perturbation} under noise sampled from $B_1$ and $B_2$, we follow the algorithm performance analysis in \citet{Chaudhuri2011}  and analyze $\mathbb{P}(J(\theta_b) - J(\theta^*) > c)$, where 
 $$J(\theta)  \;=\; \frac{1}{n}\sum_{d \in D}l(\theta; d)  + r(\theta) $$
with $l$ and $r$ as defined in Algorithm \ref{algthm:objective_perturbation},  $\theta^* = \arg\min_\theta J(\theta)$, and 
$ \theta_b  = \arg\min_\theta \left[ J(\theta)  + \frac{\phi}{\epsilon n }b^T\theta \right]\;=\; \arg\min_\theta L(\theta; b).$
That is,  $J(\theta_b) - J(\theta^*)$ measures how much the objective function deviates from the optimum due to the added noise. Given random noise $b\in\mathbb{R}^s$,
$ J(\theta_b)  + \frac{\phi}{\epsilon n }b^T\theta_b \le J(\theta^*)  + \frac{\phi}{\epsilon n }b^T\theta^*.$
Hence,
$ J(\theta_b) - J(\theta^*)  \le  \frac{\phi}{\epsilon n }b^T (\theta^* - \theta_b) 
   \le \frac{\phi}{\epsilon n } \norm{b}_2 \norm{\theta^* - \theta_b}_2$.
Let $E$ denote the event that $\{\norm{b}_2 \le \xi\}$, where $\xi = \frac{ \epsilon n}{\phi} \sqrt{ \lambda c }$. When $E$ holds, then $\frac{\phi}{\epsilon n} b^T \theta$ is $\frac{\phi \xi}{\epsilon n}$-Lipschitz. Hence, with $G(\theta) = J(\theta)$ $\lambda$-strongly convex, $ g_1(\theta) =  \frac{\phi}{\epsilon n} b^T \theta$ and $g_2 = 0$, we can invoke Lemma \ref{thm:bounded_opt_arg} to obtain $\norm{\theta^* - \theta_b}_2   \le \frac{\phi \xi}{\lambda \epsilon n}$. Therefore, when $E$ holds, then
$$ J(\theta_b) - J(\theta^*) \le \frac{\phi}{\epsilon n } \norm{b}_2 \norm{\theta^* - \theta_b}_2  
 \le \frac{\phi}{\epsilon n }  \xi \frac{\phi \xi}{\lambda \epsilon n} = c.$$
Thus 
$\mathbb{P}( J(\theta_b) - J(\theta^*) > c ) \le 1 - \mathbb{P}(E) = \mathbb{P}(\norm{b}_2  > \xi)$
when the random noise $b$ is sampled from $B_1$ or $B_2$. $\norm{B_1}_1$ is the sum of $s$ independent exponential random variables with mean $=2$ and thus $\norm{B_1}_1 \sim Gamma(s, 2)$. On the other hand, $\norm{B_2}_2 \sim \chi^2 (2s)$. But in fact $\chi^2(2s) \sim Gamma(s, 2)$. Therefore,
$\mathbb{P}( \norm{B_1}_2 > \xi)   \le  \mathbb{P}(  \norm{B_1}_1 > \xi)  
 = \mathbb{P}( \norm{B_2}_2> \xi ).$
%\[
%\begin{split}
%\mathbb{P}( \norm{B_1}_2 > \xi)   \le  \mathbb{P}(  \norm{B_1}_1 > \xi)  
% = \mathbb{P}( \norm{B_2}_2> \xi ).
%\end{split}
%\]
Thus, sampling the noise  from $B_1$ in Algorithm \ref{algthm:objective_perturbation} produces more accurate results.

%%%%%%%%%%%%%%%%%%%%%%%%%%%%%%%%%%%%%%%%%%%%%%%%%%%%%%%%%%%%%%%%%%%
%%%%%%%%%%%%%%%%%%%%%%%%%%%%%%%%%%%%%%%%%%%%%%%%%%%%%%%%%%%%%%%%%%%
%%%%%%%%%%%%%%%%%%%%%%%%%%%%%%%%%%%%%%%%%%%%%%%%%%%%%%%%%%%%%%%%%%%

\section{Application to logistic regression with elastic-net regularization}
\label{sec_4}

In this section we show how to apply the results from the previous section to penalized logistic regression. The logistic loss function $l(\theta; x, y)$ is given by
\[
l(\theta; x, y)\;=\; \log\left(1 + \exp(-y\,\theta^T x)\right),
\]
where $y \in \{-1, 1\}$. The first and second derivatives with respect to $\theta$ are
\[
\begin{split}
\nabla l(\theta; x, y) &\;=\; -\frac{1}{1 + \exp(y\,\theta^T x)}yx \\
\nabla^2 l(\theta; x, y) &\;=\; \frac{1}{1 + \exp(-y\,\theta^T x)} \frac{1}{1 + \exp(y\,\theta^T x)} x x^T. \\
%\nabla^2 l(\theta; x, y) &\;=\;\frac{1}{1 + e^{-y\,\theta^T x}} %\frac{1}{1 + e^{y\,\theta^T x}}  \left( \frac{1}{1 + e^{y\,\theta^T x}}  %- \frac{1}{1 + e^{-y\,\theta^T x}}   \right)  x x^T.  \\
\end{split}
\]
It can easily be seen that the logistic loss function satisfies the following properties: (i) $ l(\theta; x, y)$  is convex; (ii)$\nabla^2 l(\theta; x, y)$ is continuous;  and (iii) $\nabla^2 l(\theta; x, y)$    is a rank-1 matrix.
%\[
%\begin{split}
%l(\theta; x, y) & \text{ is convex}, \\
%\nabla^2 l(\theta; x, y) & \text{ is continuous},\\
%\nabla^2 l(\theta; x, y)  & \text{ is a rank-1 matrix}.
%\end{split}
%\]

We denote by $\norm{M}_1$ the nuclear norm of the matrix $M$ and we choose $\kappa$ such that  $\norm{x}_j \le \kappa$ for all $x$, where $j\in\{1,2\}$. Then 
\[
\begin{split}
\norm{\nabla^2 l(\theta; x, y)}_1 & \;\le\; \norm{ x x^T }_1 \;=\; \norm{x}_2^2 \;\le\; \norm{x}_j^2 \;\le\; \kappa^2, \quad\text{ for } j \in\{1,2\}, \\
\norm{\nabla l(\theta; x, y)}_j & \;\le\; \norm{x}_j \;\le\; \kappa,\\ 
\end{split}
\]
\noindent Thus we can apply Algorithm \ref{algthm:objective_perturbation} to output differentially private coefficients for logistic regression with elastic-net regularization. Moreover,  the logistic loss function  satisfies the conditions in Theorem \ref{thm:stability_dp} because $l(\theta; x, y)$ is Lipschitz: There exists a parameter $\theta$ such that
\[
\begin{split}
| l(\theta_1; x, y) - l(\theta_2; x, y) |  
&\;\le\;  \norm{\nabla l(\theta; x, y)}_2 \norm{\theta_1 - \theta_2}_2 \;\le\; \kappa \norm{\theta_1 - \theta_2}_2.
\end{split}
\]
Thus we can apply the stability argument in Theorem \ref{thm:stability_dp} to select the best regularization parameters in a differentially private way. In Section~\ref{sec_5} we show how well this method performs on a GWAS data set.

\section{Application to GWAS data}
\label{sec_5}

We now evaluate the performance of the proposed method based on a GWAS data set. We analyze a binary phenotype such as a disease. Each SNP can take the values 0, 1, or 2. This represents the number of minor alleles at that site. A large SNP data set is freely available from the HapMap project\footnote{http://hapmap.ncbi.nlm.nih.gov/}. It consists of SNP data from 4 populations of 45 to 90 individuals each, but does not contain any phenotypic information about the individuals. HAP-SAMPLE~\cite{Wright2007b} can be used to generate SNP genotypes for cases and controls by resampling from HapMap. This ensures that the simulated data show linkage disequilibrium (i.e., correlations among SNPs) and minor allele frequencies similar to real data. 

For our analysis we use the simulations from~\citet{Malaspinas2010}. The simulated data sets consist of 400 cases and 400 controls each with about 10,000 SNPs per individual (SNPs were typed with the Affymetrix CHIP on chromosome 9 and chromosome 13 of the Phase I/II HapMap data). For each data set two SNPs with a given minor allele frequency (MAF) were chosen to be causative. We will analyze the results for minor allele frequency (MAF) $= 0.25$. The simulations were performed under the multiplicative effects model: Denoting the two causative SNPs by $X$ and $Y$ and the disease status by $D$ (i.e., $X, Y\in\{0,1,2\}$ and $D\in\{-1,1\}$, where $1$ describes the diseased state), then the multiplicative effects model can be defined through the odds of having a disease:
$$\frac{\mathbb{P}(D=1\mid X, Y)}{\mathbb{P}(D=-1\mid X, Y)} \quad=\quad \epsilon\, \alpha^{X} \beta^{Y} \delta^{XY}.$$ 
%$$\frac{\mathbb{P}(D=1\mid X, Y)}{\mathbb{P}(D=-1\mid X, Y)}: \qquad \begin{tabular}{l c | c c c} &&& $X$ &\\ && 0 & 1 & 2\\ \hline  & 0 & $\epsilon$ & $\epsilon\beta$& $\epsilon\beta^2$\\ $Y\;$ & 1 & $\epsilon\alpha$ & $\epsilon\alpha\beta\delta$& $\epsilon\alpha\beta^2\delta^2$ \\ & 2 & $\epsilon\alpha^2$ & $\epsilon\alpha^2\beta\delta^2$ & $\epsilon\alpha^2\beta^2\delta^4$\end{tabular}$$
This model corresponds to a log-linear model with interaction between the two SNPs. For our simulations we chose $\epsilon = 0.64$, $\alpha = \beta = 0.91$ and $\delta = 2.73$. This results in a sample disease prevalence of 0.5 and effect size of 1, which are typical values for association studies. See~\citet{Malaspinas2010} for more details. 

In the first step, we screen all SNPs and select a subset of SNPs with the highest $\chi^2$-scores based on a simple $\chi^2$-test for association between each single SNP and the phenotype. Various approaches for performing the screening in a differentially private manner were discussed and analyzed in \citet{Uhler2013, Johnson2013, Yu2014}; We concentrated on the second step and did not employ the differentially private screening approaches in this paper. The second step of the two-step procedure consists of performing penalized logistic regression with elastic-net regularization on the selected subset of SNPs and choosing the best regularization parameters in a differentially private manner. In the following, we analyze the statistical utility of the second step and show how accurately our end-to-end differentially private penalized logistic regression method is able to detect the causative SNPs and their interaction. 

The elastic-net penalty function has the  form $\frac{1}{2}\lambda (1-\alpha) \ell_2 + \lambda \alpha \ell_1,$ where
$\alpha$ controls the sparsity of the resulting model and $\lambda$ controls the extent to which the elastic-net penalty affects the loss function.  In the simulation, we apply a threshold criterion to the terms in the model so that  we exclude from the model the $i$th term if  its regression coefficient, $\theta_i$, satisfies  $|\theta_i| / \max\limits_i\{|\theta_i | \} < r$, where $\max\limits_i\{|\theta_i | \}$ is the largest coefficient in absolute value and $r$ is a thresholding ratio, which we set to 0.01.

In our experiments, we selected $M=5$ SNPs with the highest $\chi^2$-scores, which include the two causative SNPs, for further analysis. We denote by $\epsilon$ the privacy budget, by $\alpha$ the sparsity parameter in the elastic-net penalty function, and by ``$\text{\emph{convex\_min}}$'' the condition of strong convexity imposed on the objective function (see Theorem \ref{thm:stability_dp}). Note that $\text{\emph{convex\_min}}$ is a function of $M$ and $\epsilon$. For elastic-net with $\alpha$ fixed, we need the smallest candidate parameter $\lambda_{min} \ge \text{\emph{convex\_min}} / (1-\alpha)$.

In Figure~\ref{fig:ru_barchart}, we analyze the sensitivity of our method. For different sparsity parameters $\alpha$ and different privacy budgets $\epsilon$, which determine $\text{\emph{convex\_min}}$ given a fixed $M$, we show how often, out of 100 simulations each, our algorithm recovered the interaction term (leftmost bar in red), the main effects scaled by a factor of $1/2$ to account for the two main effects (middle bar in green) and all effects, i.e.~the interaction effect and the two main effects (rightmost bar in blue). As the privacy budget $\epsilon$ increases, the amount of noise added to the regression problem decreases, and hence the frequency of selecting the correct effects in the regression analysis increases. The plots also show that as the sparsity parameter $\alpha$ increases, the frequency of selecting the correct terms decreases.

In Figure~\ref{fig:specificity_barchart} we analyze the specificity of our method. For different sparsity parameters $\alpha$  and different strong convexity conditions $\text{\emph{convex\_min}}$, we show how often, out of 100 simulations each, our algorithm did not include any additional effects in the selected model. As $\alpha$ increases, the selected model becomes sparser and the algorithm is hence less likely to wrongly include additional effects. We also observe that as $\text{\emph{convex\_min}}$ decreases, the specificity increases. This can be explained by how we choose the candidate parameters $\lambda$, namely as multiples of the smallest allowed value for $\lambda$, which is $\text{\emph{convex\_min}} / (1-\alpha)$. When $\lambda$ is smalll, the effect of the penalty terms diminishes, and we are essentially performing a regular logistic regression, which does not produce sparse models.

%In Figure~\ref{fig:specificity_barchart} we plotted the frequency of correctly excluding the wrong SNPs and wrong interaction terms from the model. This reflects the specificity of our method. As $\alpha$ increases, the model becomes sparser and less likely to include a wrong SNP or a wrong interaction terms. We also observer that as $\epsilon$ increases, the frequency of  correctly excluding the wrong regression terms decreases. This is due to the way we choose the candidate $\lambda$ parameters, which are multiples of the smallest allowed value for $\lambda$: as $\epsilon$ increases, the smallest allowed value for $\lambda$ decreases. When $\lambda$ is small,  the effect of the penalty terms diminishes, and we are essentially performing a regular logistic regression, which does not produce sparse models.

In Figure~\ref{fig:no_noise_barchart}, we plotted the results of non-private penalized logistic regression with elastic-net penalty to contrast Figure~\ref{fig:ru_barchart} and Figure~\ref{fig:specificity_barchart}. The results of the non-private penalized logisitc regression is indirectly related to $\epsilon$ because the choice of the smallest regularization parameter $\lambda$ is bounded below by $\text{\emph{convex\_min}} / (1-\alpha)$ and $\text{\emph{convex\_min}}$ is a function of $\epsilon$. 
We can observe from Figure~\ref{fig:no_noise_barchart}  that when the regularization parameter $\lambda$ is large (i.e., $\text{\emph{convex\_min}} \ge 1.58$), the regression analysis screens out all effects. Hence, the sensitivity is 0 and the specificity is 1. When $\lambda$ is small  (i.e., $\text{\emph{convex\_min}} \le 0.18$),  the amount of regularization also becomes marginal, and we begin to see that the sensitivity increases but the specificity decreases. 
Figure~\ref{fig:no_noise_barchart} shows that we can identify the correct model when $\alpha=0.1$ and $\text{\emph{convex\_min}}=0.18$. In contrast, when we use the same $\alpha$ and  $\text{\emph{convex\_min}}$ for differentially private regressions, Figure~\ref{fig:ru_barchart} shows that we can obtain a good sensitivity result, but  Figure~\ref{fig:specificity_barchart} shows that the specificity result for this choice is poor.

\begin{figure}
\caption{\label{fig:ru_barchart}\small Sensitivity analysis for different sparsity parameters $\alpha$,  privacy budgets $\epsilon$, and strong convexity conditions $\text{\emph{convex\_min}}$ when the top 5 SNPs are used for the analysis: the red (leftmost) bar shows how often, out of 100 simulations each, the algorithm recovered the interaction term, the green (middle) bar corresponds to the main effects scaled by a factor of $1/2$ and the rightmost (blue) bar corresponds to all effects, i.e.~2 main effects and 1 interaction effect. }
\centering
\includegraphics[width=0.6\textwidth, page=1]{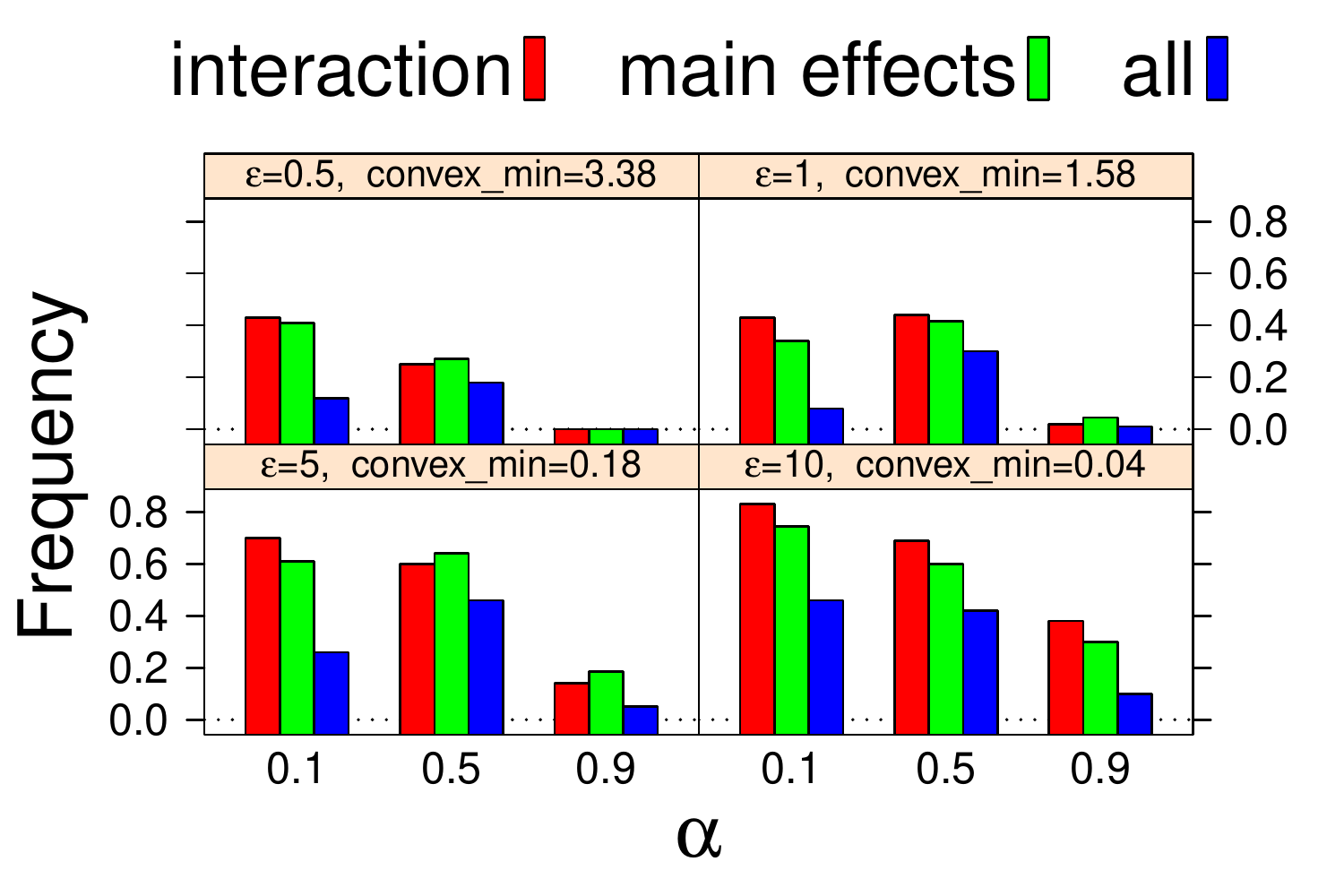}
\end{figure}

\begin{figure}
\caption{\label{fig:specificity_barchart}\small Specificity analysis for different sparsity parameters $\alpha$ and strong convexity conditions $\text{\emph{convex\_min}}$:  the plot shows how often, out of 100 simulations each, our algorithm did not include any additional effects in the selected model.
%Each bar represents the frequency that the model correctly exclude the wrong SNPs. $\alpha$ is a parameter in the elastic-net penalty function $\frac{1}{2}\lambda (1-\alpha) \ell_2 + \lambda \alpha \ell_1$, which controls the amount of sparseness. {\it epsilon} is the privacy budget $\epsilon$. 
}
\centering
\includegraphics[width=0.6\textwidth, page=2]{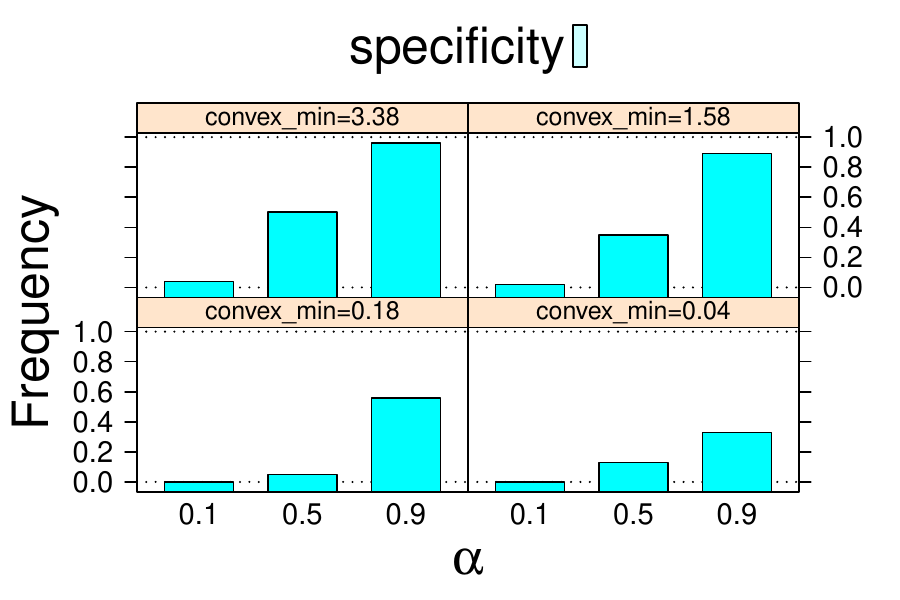}
\end{figure}

\begin{figure}
\vspace{-1em}
\caption{\label{fig:no_noise_barchart}\small Results of non-private logistic regression with elastic-net penalty.  Figure \ref{fig:no_noise_ru_barchart} and Figure \ref{fig:no_noise_specificity_barchart} would be compared with Figure \ref{fig:ru_barchart} and  Figure \ref{fig:specificity_barchart}, respectively. }
\centering
\begin{subfigure}[b]{0.45\textwidth}
\caption{\label{fig:no_noise_ru_barchart}Sensitivity}
\includegraphics[width=\textwidth, page=1]{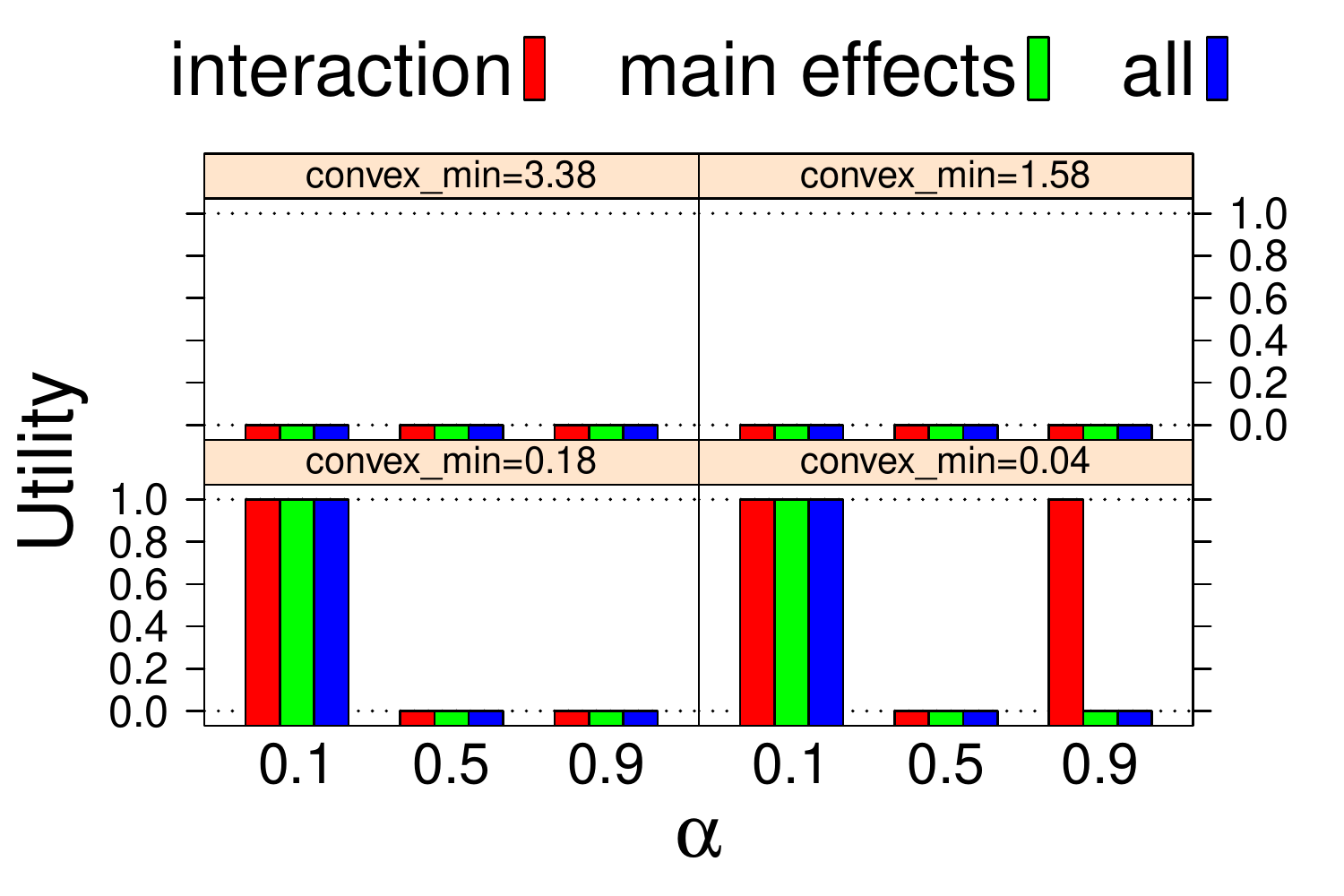}
\end{subfigure} %
~
\begin{subfigure}[b]{0.45\textwidth}
\caption{\label{fig:no_noise_specificity_barchart}Specificity}
\includegraphics[width=\textwidth, page=2]{RU_barchart_with_no_noise}
\end{subfigure} %
\vspace{-1em}
\end{figure}

\section{Conclusions}
\label{sec_6}

Various papers have argued that it is possible to use aggregate genomic data to compromise the privacy of individual-level information collected in GWAS databases. In this paper, we respond to these attacks by proposing a new method to release regression coefficients from association studies that satisfy differential privacy and hence come with privacy guarantees against arbitrary external information.

By extending the approaches in \citet{Chaudhuri2013} and \citet{Kifer2012} we developed an end-to-end differentially private procedure  for solving regression problems with convex penalty functions including selecting the penalty parameters by cross-validation. We also provided the exact form of the random noise used in the objective function perturbation mechanism and showed that the perturbation noise can be efficiently sampled.

As a special case of a regression problem, we focused on penalized logistic regression with elastic-net regularization, a method widely used to perform GWAS analyses and identify disease-causing genes. Our simulation results in Section \ref{sec_5} showed that our method is applicable to GWAS data sets and enables us to perform data analysis that preserves privacy and utility. The risk-utility analysis about the tradeoff between privacy ($\epsilon$) and utility (correctly identifying the causative SNPs) helps us decide on the appropriate level of privacy guarantee for the released data. We hope that approaches such as those described in this paper will allow the release of more information from GWAS going forward and allay the privacy concerns that others have voiced over the past decade.

\printbibliography

%%%%%%%%%%%%%%%%%%%%%%%%%%%%%%%%%%%%%%%%%%%%%%%%%%%%%%%%%%%%%%%%%%%
%%%%%%%%%%%%%%%%%%%%%%%%%%%%%%%%%%%%%%%%%%%%%%%%%%%%%%%%%%%%%%%%%%%
%%%%%%%%%%%%%%%%%%%%%%%%%%%%%%%%%%%%%%%%%%%%%%%%%%%%%%%%%%%%%%%%%%%

\appendix
\small
\section{Proofs}

%%%%%%%%%%%%%%%%%%%%%%%%%%%%%%%%%%%%%%%%%%%%%%%%%%%%%%%%%%%%%%%%%%%
%%%%%%%%%%%%%%%%%%%%%%%%%%%%%%%%%%%%%%%%%%%%%%%%%%%%%%%%%%%%%%%%%%%

\subsection{Proof of Theorem \ref{thm:stability_dp}}
\label{pf:stability_dp}
\begin{lemma}
\label{thm:bounded_opt_arg}
Let $G$, $g_1$, and $g_2$ be  vector-valued continuous functions.  Suppose that $G$ is $\lambda$-strongly convex,  $g_1$ is   convex and $\gamma_1$-Lipschitz, and $g_2$ is   convex and $\gamma_2$-Lipschitz. If $f_1 = \arg\min_{f} (G + g_1) (f)$ and   $f_2 = \arg\min_{f} (G + g_2) (f)$,
then $\norm{f_1 - f_2}_2 \le (\gamma_1 + \gamma_2) / \lambda.$
%\[
%\norm{f_1 - f_2}_2 \le (\gamma_1 + \gamma_2) / \lambda.
%\]
\end{lemma}
\begin{proof}[of Lemma \ref{thm:bounded_opt_arg}]
$G+g_1$ and $G + g_2$ are $\lambda$-strongly convex because $G$ is $\lambda$-strongly convex and $g_1$ and $g_2$ are convex.
Then for $ j, k, w \in \{1, 2\}$, $j \ne k$,
\[
(G+g_w)(f_j) \ge (G+g_w)(f_k) + \partial (G+g_w)(f_k) ^T (f_j - f_k) + \frac{\lambda}{2} ||f_j - f_k||^2
\]
where $\partial (G+g_w)$ denotes the subgradient. We know that $0 \in \partial (G+g_w)(f_w)$ because $f_w$ minimizes $G+g_w$. Hence,
\[
\begin{split}
& (G+g_2)(f_1) \;\ge\; (G+g_2)(f_2) + \frac{\lambda}{2} ||f_1 - f_2||_2^2, \\
&  (G+g_1)(f_2) \;\ge\; (G+g_1)(f_1) + \frac{\lambda}{2} ||f_1 - f_2||_2^2. \\
\end{split}
\]
By summing these two inequalities we obtain
$$
(G+g_2)(f_1) + (G+g_1)(f_2) \; \ge\; (G+g_2)(f_2) + (G+g_1)(f_1)  + \lambda ||f_1 - f_2||_2^2
$$
and hence
$
[g_2(f_1) - g_2(f_2) ] + [ g_1(f_2) - g_2 (f_1) ] \;\ge\; \lambda \norm{f_1 - f_2}_2^2.
$
The fact that $g_w$ is $\gamma_w$-Lipschitz implies that 
$  \bigg|g_2(f_1) - g_2(f_2)\bigg| + \bigg|g_1(f_2) - g_2 (f_1)\bigg| \le  (\gamma_1 + \gamma_2) \norm{f_1 - f_2}_2$ and hence 
\[
\begin{split}
\lambda ||f_1 - f_2||_2^2 
&\le [g_2(f_1) - g_2(f_2) ] + [ g_1(f_2) - g_2 (f_1) ] \\
&\le \bigg|g_2(f_1) - g_2(f_2)\bigg| + \bigg|g_1(f_2) - g_2 (f_1)\bigg|  \le   (\gamma_1 + \gamma_2) ||f_1 - f_2||_2.
\end{split}
\]
Therefore
$
\norm{f_1 - f_2}_2 \le (\gamma_1 + \gamma_2)  / \lambda.
$
\qed
\end{proof}

%%%%%%%%%%%%%%%%%%%%%%%%%%%%%%%%%%%%%%%%%%%%%%%%%%%%%%%%%%%%%%%%%%%
%%%%%%%%%%%%%%%%%%%%%%%%%%%%%%%%%%%%%%%%%%%%%%%%%%%%%%%%%%%%%%%%%%%

\begin{proof}[of Theorem \ref{thm:stability_dp}]
For notational convenience we assume that $c_{\min} \ge c^*$ so that $$L(\theta; T) \;=\; \frac{1}{n} \sum\limits_{d \in T}^{} l(\theta; d) + \lambda^T r(\theta) + \frac{\phi}{\epsilon n} b^T  \theta.$$
If $c_{\min} < c^*$, we can extend $r$ to include $r_{t+1}(\theta) = \frac{\max\{0, c^* - c_{\min}\}}{2} \norm{\theta}_2^2$ and extend each $\lambda \in \Lambda$ such that  $\lambda_{t+1}=1$. 
First, we show that $|q( \theta^*(T) , V) - q(\theta^*(T') , V  )| \le  \beta_1/n$ for training sets $T$ and $T'$ that differ only by one record. Here,
$
\theta^*(T) \;=\; \arg\min_{\theta} L(\theta; T).
$
Let $d = T\backslash T'$,  $d' = T'\backslash T$, 
\[
\begin{split}
G(\theta; T, T')  &\;=\; \frac{1}{n}\sum\limits_{d \in T \cap T'}^{} l(\theta; d)  + \lambda^T r(\theta)   + \frac{\phi}{\epsilon n}b^T\theta, \\
g_1(\theta; T, T') &\;=\; \frac{1}{n} l (\theta; d) \qquad \text{ and } \qquad g_2(\theta; T, T') \;=\;\frac{1}{n}  l (\theta; d').
\end{split}
\]
%\[
%\begin{split}
%G(\theta; T, T')  &\;=\; \frac{1}{n}\sum\limits_{d \in T \cap T'}^{} l(\theta; d)  + \lambda^T r(\theta)   + \frac{\phi}{\epsilon n}b^T\theta, \\
%g_1(\theta; T, T') &\;=\; \frac{1}{n} l (\theta; d) , \text{ where } d = T\backslash T',\\
%g_2(\theta; T, T') &\;=\;\frac{1}{n}  l (\theta; d') , \text{ where } d' = T'\backslash T.\\
%\end{split}
%\]
Then $G$ is $c_{\min} $-strongly convex, and $g_1$ and $g_2$ are convex and $\gamma/n$-Lipschitz. 
By Lemma~\ref{thm:bounded_opt_arg}, 
$\norm{\theta^*(T) - \theta^*(T')}_2 \;\le\; \frac{2\gamma}{n c_{\min} }.$
Since $h$ is $\kappa$-Lipschitz we obtain for any validation set $V$, $|q( \theta^*(T), V) - q( \theta^*(T'), V  )|\; \le\;  \frac{2\gamma \kappa}{n c_{\min}}.$ %
%\[
%|q( \theta^*(T), V) - q( \theta^*(T'), V  )|\; \le\;  \frac{2\gamma \kappa}{n c_{\min}}.
%\]

Second, we show that for all $\lambda \in \Lambda$ and for all validation sets $V$ and $V'$ that differ in a single record, 
$|q( \theta^*(T), V) - q( \theta^*(T'), V'  )| \;\le\;  \beta_2 / m$.
Since $h$ is non-negative, $|q( \theta^*(T), V) - q(\theta^*(T'), V'  )| \; \le\; h_{\max} /m,$
%\[
%|q( \theta^*(T), V) - q(\theta^*(T'), V'  )| \; \le\; h_{\max} /m,
%\]
where $ h_{\max} = \sup_d h(\theta^*(T); d)$. By definition, $h_{\max}  \le h^*$. Moreover, because $h$ is $\kappa$-Lipschitz, $h_{\max}  \le \kappa \norm{\theta^*(T)}_2$. So $h_{\max}  \le \min  \{ h^*, \kappa \norm{\theta^*(T)}_2 \}$. Now let $E$ be the event that $\norm{b}_2 \le \xi$. % (note that $||b||_1 \le \xi \Rightarrow ||b||_2 \le \xi $ ).  If $b$ is drawn from $s$ Laplace random variables, then by  Lemma 17 of \citet{Chaudhuri2011}, $P(||b||_1 \le 2 s \log(sk/\delta)) \ge 1-\delta / k$; if $b = \frac{W}{||W||_2}  Y_{2s}$, then $||b||_2 =Y_{2s}$, and by Lemma 1 of \citet{Laurent2000}, $P\left(||b||_2 \le  \left( \sqrt{s} + \sqrt{ \log (k / \delta) } \right)^2 + \log(k / \delta)  \right)  \ge 1-\delta / k$. Thus, provided that $E$ holds, we have 
Provided that $E$ holds, we have
$
| b^T \theta_1 - b^T \theta_2|  \le \norm{b}_2 \norm{\theta_1 - \theta_2}_2 \;\le\; \xi \norm{\theta_1 - \theta_2}_2.
$
Let $
G(\theta) = \lambda^T r(\theta)$, $g_1(\theta; T) =  \frac{1}{n} \sum_{d\in T}^{} l (\theta; d) + \frac{\phi}{\epsilon n} b^T \theta$, and $g_2(\theta) = 0$.
Then $G$ is $c_{\min}$-strongly convex, $g_1$ is $\left(\gamma +  \frac{\phi \xi }{\epsilon n}\right)$-Lipschitz, and $g_2$ is $0$-Lipschitz. Since $G+g_2$ is minimized when $\theta = 0$, we obtain by invoking Lemma \ref{thm:bounded_opt_arg} that
$
\norm{ \theta^*(T) }_2 \;=\; \norm{ \theta^*(T) - 0}_2 \;\le\; \frac{1 }{c_{\min}} \left(\gamma +  \frac{\phi \xi }{\epsilon n} \right)
$
Therefore, $|q( \theta^*(T) , V) - q( \theta^*(T) , V' )| \;\le\;  \frac{1 }{m} \min\left\{h^*, \frac{\kappa}{c_{\min} }\left(\gamma +  \frac{\phi \xi }{\epsilon n} \right)\right\}.$
%\[
%|q( \theta^*(T) , V) - q( \theta^*(T) , V' )| \;\le\;  \frac{1 }{m} \min\left\{h^*, \frac{\kappa}{c_{\min} }\left(\gamma +  \frac{\phi \xi }{\epsilon n} \right)\right\}.
%\]
\qed
\end{proof}

\subsection{Proof of Theorem \ref{thm:objective_perturbation}}
\label{proof_alg}
\begin{lemma}
\label{thm:mat_det_eigen}
If $A$ is of full rank and $E$ has rank at most 2, then 
\[
\frac{\det(A+E) - \det(A) }{ \det(A) } 
\;=\; \lambda_1 (A^{-1} E ) + \lambda_2(A^{-1}E) + \lambda_1(A^{-1}E) \lambda_2(A^{-1}E),
\]
where $\lambda_j (Z)$ denotes the $j$-th eigenvalue of matrix $Z$.
\end{lemma}

\begin{proof}[of Lemma \ref{thm:mat_det_eigen}] 
See Lemma 10 in \citet{Chaudhuri2011}.
\end{proof}

\begin{proof}[of Theorem \ref{thm:objective_perturbation}] 
Similar to the proof by \citet{Chaudhuri2011}, we show that  if  $r$ is infinitely differentiable, then Algorithm \ref{algthm:objective_perturbation} is $\epsilon$-differentially private. It then follows from the successive approximation method by \citet{Kifer2012} that Algorithm \ref{algthm:objective_perturbation} is still $\epsilon$-differentially private even if $r$ is convex but not necessarily differentiable.

Let $g$ denote the probability density function of the algorithm's output $\theta^*$. Our goal is to show that
$
e^{-\epsilon} 
	\le \frac{g(\theta|D)  }{g(\theta|D')} 
	\le e^{\epsilon}.
$
Suppose that the Hessian of $r$ is continuous. Because $0 = \nabla L(\theta; D)$, we have 
\[
T_D(\theta) := b \;=\; - \frac{\epsilon}{\phi} \left[ \sum\limits_{d \in D}^{}    \nabla l(\theta; d) + n\nabla r(\theta)  \right] \; \text{ and }\;
\nabla T_D(\theta)  \;=\;    - \frac{\epsilon}{\phi} \left[ \sum\limits_{d \in D}^{}   \nabla^2 l(\theta; d)+ n\nabla^2 r(\theta)   \right].\\
\]
%\[
%\begin{split}
%T_D(\theta) := b &\;=\; - \frac{\epsilon}{\phi} \left[ \sum\limits_{d \in D}^{}    \nabla l(\theta; d) + n\nabla r(\theta)  \right] \\
%\nabla T_D(\theta)  &\;=\;    - \frac{\epsilon}{\phi} \left[ \sum\limits_{d \in D}^{}   \nabla^2 l(\theta; d)+ n\nabla^2 r(\theta)   \right].\\
%\end{split}
%\]

$T_D$ is injective because $L(\theta; D)$ is strongly convex. Also, $T_D$ is  continuously differentiable. Therefore, 
\[
	\begin{split}
	\frac{g(\theta|D)  }{g(\theta|D')}  
		& \;=\;\frac{ f (T_D(\theta))}{ f (T_{D'}(\theta))}\,\frac{ |\!\det(\nabla T_D) ( \theta)|  }{|\!\det(\nabla T_{D'}) ( \theta)| },
	\end{split}
\]
where $f$ is the density function of $b$. 
%\todo[inline, size=\scriptsize]{In \citet{Chaudhuri2011} and \citet{Kifer2012}, the authors wrote
%		$\frac{g(\theta|D)  }{g(\theta|D')}  
%		= \frac{ f (T_D(\theta))  |\!\det(\nabla T_{D}) ( \theta)\!| ^{-1}  }{ f (T_{D'}(\theta))  |!\det(\nabla T_{D'}) ( \theta)\!| ^{-1}  }$
%	}
%
%

We first consider $\frac{  |\det(\nabla T_D) ( \theta)|  }{ |\det(\nabla T_{D'}) ( \theta)| }$.
Let 
$A =- \frac{\phi}{\epsilon} \nabla T_{D'}$, $E = \nabla^2 l(\theta; D \backslash D') - \nabla^2 l(\theta; D' \backslash D).$
Because $l$ is convex and $r$ is strongly convex,   $\nabla T_D( \theta)$ is positive definite. Hence, $A$ has full rank. Also, $E$ has rank at most 2 because  $\nabla^2 l(\theta; d)$ is a rank 1 matrix by assumption. By Lemma \ref{thm:mat_det_eigen},
\[
\begin{split}
\frac{ |\det(\nabla  T_D( \theta) )|  }{   |\det(\nabla  T_{D'}( \theta) )|  } 
	&\;=\;\left| \frac{ \det(A+E)  }{ \det(A) }  \right|  \;\le\; 1 + s_1(A^{-1}E) +  s_2(A^{-1}E) +  s_1(A^{-1}E)  s_2(A^{-1}E),
\end{split}
\]
where $s_i(M)$ denotes the $i$th largest singular value of $M$. 
Because $r$ is $\lambda$-strongly convex, the smallest eigenvalue of $A$ is at least $n\lambda$. So $s_i(A^{-1}E) \le \frac{  s_i(E) }{  n \lambda}$. Because   $\norm{\nabla l(\theta; d) }_j \le \kappa$ for $j \in \{1, 2\}$,  applying the triangle inequality to the nuclear norm yields $s_1(E) + s_2(E) \;\le\; \norm{ \nabla^2 l (\theta; D \backslash D'  ) }_1 + \norm{ \nabla^2 l (\theta; D' \backslash D )}_1  \le 2c.$
Therefore,  $s_1(A^{-1}E)  \,s_2(A^{-1}E) \le \left( \frac{c }{n\lambda} \right)^2$, and
\[
\frac{ |\det(\nabla  T_D)( \theta)|  }{   |\det(\nabla  T_{D'})( \theta) |  } 
	\;=\;\frac{ |\det(A+E) | }{ |\det(A)| } \;\le\;  \left( 1 + \frac{ c  }{n \lambda}\right)^2.
\]

Now we consider 
$\frac{ f (T_{D}(\theta))  }{  f ( T_{D'}(\theta) )    }$. Since
\[
\begin{split}
\norm{T_{D}(\theta) - T_{D'}(\theta)}_j
&\;=\; \left(\frac{\epsilon}{\phi}\right) \norm{  \nabla l(\theta; D \backslash D') -   \nabla  l(\theta; D' \backslash D)  }_j \\
&\;\le\;  \left(\frac{\epsilon}{\phi}\right) \left( \norm{  \nabla l(\theta; D \backslash D')}_j  + \norm{ \nabla l(\theta; D' \backslash D) }_j\right) \;\le\; \frac{2 \kappa \epsilon}{\phi},
\end{split}
\]
we obtain
$
\frac{f ( T_{D}(\theta) )  }{  f ( T_{D'}(\theta) )  }
	= \exp\left( -\frac{\norm{T_{D}(\theta)}_j  }{2 } \right) \bigg/ \exp\left(  \frac{-\norm{T_{D'}(\theta)}_j }{ 2}  \right) 
	\le \exp\left( \frac{ \kappa  \epsilon }{\phi} \right),
$
%$
%\frac{f ( T_{D}(\theta) )  }{  f ( T_{D'}(\theta) )  }
%	=  \frac{ \exp\left( -\frac{\norm{T_{D}(\theta)}_j  }{2 } \right) }{ \exp\left(  \frac{-\norm{T_{D'}(\theta)}_j }{ 2}  \right) }
%	\le \exp\left( \frac{ \kappa  \epsilon }{\phi} \right),
%$
%\[
%\frac{f ( T_{D}(\theta) )  }{  f ( T_{D'}(\theta) )  }
%	\;=\;  \frac{ \exp\left( -\frac{\norm{T_{D}(\theta)}_j  }{2 } \right) }{ \exp\left(  \frac{-\norm{T_{D'}(\theta)}_j }{ 2}  \right) }
%	\;=\;  \exp\left(\frac{\norm{T_{D'}(\theta)}_j   - \norm{T_{D}(\theta)}_j }{ 2^{}} \right)
%	\;\le\; \exp\left( \frac{ \kappa  \epsilon }{\phi} \right),
%\]
and therefore,
\[
\frac{ f (T_D(\theta))}{ f (T_{D'}(\theta)) }\,\frac{ |\!\det(\nabla T_D)( \theta)|  }{ |\!\det(\nabla T_{D'})( \theta)| }
\;\le\; \exp\left( \frac{ \kappa  \epsilon }{\phi} + 2\log\left( 1 + \frac{c}{n  \lambda}\right)  \right)  \;\le \;e^\epsilon.
\]
\end{proof}

\subsection{Proof of Proposition \ref{thm:b_dist_exact}}
\label{pf:b_dist_exact} 
\begin{proof}[of Proposition \ref{thm:b_dist_exact}]
The distribution of $X$ is a special case of an $s$-dimensional \emph{power exponential distribution} as defined by \citet{Gomez1998a}, namely $X\sim PE_s(\mu,\Sigma,\beta)$ with $\mu=(0,\dots ,0)^T$,  $\Sigma = \textrm{Id}_s$ and $\beta=\frac{1}{2}$. \citet{Gomez1998a} proved that if $T\sim PE_{s}(\mu,\Sigma,\beta)$, then $T$ has the same distribution as
%$$\mu + YA^TZ,$$
$\mu + YA^TZ,$ where $Z$ is a random vector with uniform distribution on the unit sphere in $\mathbb{R}^s$, $Y$ is an absolutely continuous non-negative random variable, independent from $Z$, whose density function is
$$g(y) = \frac{s}{\Gamma\left(1+\frac{s}{2\beta}\right)2^{\frac{s}{2\beta}}} y^{s-1}\exp\left(-\frac{1}{2}y^{2\beta}\right) I_{(0,\infty)}(y),$$
and $A\in\mathbb{R}^{s\times s}$ is a square matrix such that $A^TA = \Sigma$.

Note that for $\beta = \frac{1}{2}$, the distribution of $Y$ boils down to a $\chi^2$-distribution with $2s$ degrees of freedom. In addition, if $W\sim\mathcal{N}(0,\textrm{Id}_s)$, then $W/|\!|W|\!|$ is uniformly distributed on the unit $s$-sphere. Finally, since $\Sigma = \textrm{Id}_s$ we get that $A= \textrm{Id}$.
\qed
\end{proof}

\end{document}